\documentclass[11pt]{scrartcl}
\usepackage[utf8]{inputenc}

\usepackage{amsthm,amsmath,amssymb,natbib,graphicx,url,algorithm2e,hyperref} %
\usepackage{stmaryrd,color}
\usepackage[textwidth=2.5cm,disable]{todonotes}

\newcommand{\mrrand}[1]{\todo[color=red!40]{\footnotesize\textbf{MR:} #1}}

\newcommand{\mgrand}[1]{\todo[color=blue!40,size=\footnotesize]{\textbf{MG:} #1}}

\newcommand{\smid}{\mathbin{;}}

\renewcommand{\phi}{\varphi}
\newcommand{\err}{\operatorname{err}}

\newtheoremstyle{mythm}%
  {}%
  {}%
  {\itshape}%
  {}%
  {\bfseries}%
  {.}%
  {.5em}%
  {\thmname{#1}~\thmnumber{#2}\ifthenelse{\equal{\thmnote{#3}}{}}{}{~(\thmnote{#3})}}%

\newtheoremstyle{mydefn}%
  {}%
  {}%
  {\upshape}%
  {}%
  {\bfseries}%
  {.}%
  {.5em}%
  {\thmname{#1}~\thmnumber{#2}\ifthenelse{\equal{\thmnote{#3}}{}}{}{~(\thmnote{#3})}}%

\newtheoremstyle{myremark}%
  {}%
  {}%
  {\upshape}%
  {}%
  {\itshape}%
  {.}%
  {.5em}%
  {\thmname{#1}~\thmnumber{#2}\ifthenelse{\equal{\thmnote{#3}}{}}{}{~(\thmnote{#3})}}%

\newtheorem{theorem}{Theorem}[section]
\newtheorem{lemma}[theorem]{Lemma}

\theoremstyle{mydefn}
\newtheorem{definition}[theorem]{Definition}
\newtheorem{example}[theorem]{Example}
\theoremstyle{myremark}

\theoremstyle{mythm}

\begin{document}

\title{Learning MSO-definable hypotheses on strings}
\author{\large Martin Grohe\\\normalsize RWTH Aachen
  University\\[-0.8ex]\normalsize grohe@informatik.rwth-aachen.de
\and \large Christof L{\"o}ding\\\normalsize RWTH Aachen
  University\\[-0.8ex]\normalsize loeding@informatik.rwth-aachen.de
\and\large Martin Ritzert\\\normalsize RWTH Aachen
  University\\[-0.8ex]\normalsize ritzert@informatik.rwth-aachen.de}
\date{}
\maketitle

\begin{abstract}
We study the classification problems over string data for hypotheses
specified by formulas of monadic second-order logic MSO. The goal is
to design learning algorithms that run in time polynomial in the size
of the training set, independently of or at least sublinear in the size
of the whole data set. We prove negative as well as positive
results. If the data set is an unprocessed string to which our
algorithms have local access, then learning in sublinear time is
impossible even for hypotheses definable in a small fragment of
first-order logic. If we allow for a linear time pre-processing of the
string data to build an index data structure, then learning of
MSO-definable hypotheses is possible in time polynomial in the size
of the training set, independently of the size of the whole data set.
\end{abstract}

\section{Introduction}

We study classification problems in a declarative framework \citep[introduced in][]{grohe2004learnability,grorit17} where instances are elements or tuples of elements of some background structure and hypotheses are specified by formulas of a suitable logic, using parameters (or constants) from the background structure. The background structure, say $B$, captures properties of and relations between data points and more generally all kinds of structural information about the data. Over this background structure we can specify a parametric model by a formula $\phi(\bar x\smid\bar y)$ of some logic L, which has two types of free variables, the \emph{instance variables} $\bar x=(x_1,\ldots,x_k)$ and the \emph{parameter variables} $\bar y=(y_1,\ldots,y_\ell)$. Then \emph{instances} of our classification problem are tuples $\bar u\in U(B)^k$, where $U(B)$ denotes the universe of $B$. For each choice $\bar v\in U(B)^\ell$ of \emph{parameters}, the formula defines a function $\llbracket\phi(\bar x\smid\bar v)\rrbracket^B:U(B)^k\to\{0,1\}$ by 
\begin{equation*}
  \label{eq:1}
  \llbracket\phi(\bar x\smid\bar v)\rrbracket^B(\bar u):=
\begin{cases}
  1&\text{if }B\models\phi(\bar u\smid\bar v),\\
  0&\text{otherwise},
\end{cases}
\end{equation*}
where $B\models\phi(\bar u\smid\bar v)$ denotes that the structure $B$ satisfies $\phi$ if the instance variables $\bar x$ are interpreted by $\bar u$ and the parameter variables $\bar y$ by $\bar v$. We regard  $\llbracket\phi(\bar x\smid\bar v)\rrbracket^B$ as a hypothesis over the instance space $U(B)^k$, which we want to generate from a training set of labeled examples $(\bar u_i,\lambda_i)\in U(B)^k\times\{0,1\}$.\footnote{As such, this framework only allows it to describe binary classification problems, but it is easy to extend it to general classification problems.} We call hypotheses of the form $\llbracket\phi(\bar x\smid\bar v)\rrbracket^B$ for an L-formula $\phi(\bar x\smid\bar y)$ L-\emph{definable hypotheses} over $B$.

In this paper, background structures are \emph{strings}, which may model text data, but also traces of program executions, DNA sequences, transaction sequences, and in general streams of symbolic data. The logic L that we use to define our models is \emph{monadic second-order logic} MSO, which may be the best studied logic for strings and is closely related to finite automata \citep[see][]{tho97a}. Some of our results, in particular the lower bounds, hold for fragments of MSO such as first-order logic FO and even the existential and quantifier-free fragments of FO.

Within this framework, we may study two kinds of algorithmic problems, \emph{parameter learning} (or \emph{parameter estimation}), where we regard the formula $\phi(\bar x\smid \bar y)$ as
fixed and try to find parameters $\bar v$ that fit the data, and \emph{model learning} (or \emph{model estimation}), where we want to find a suitable formula $\phi(\bar x\smid\bar y)$ and parameters $\bar v$. Our algorithms follow an empirical risk minimization paradigm; for the model learning problem, we bound the quantifier rank of the formula $\phi(\bar x\smid\bar y)$ to avoid overfitting. Hence the algorithmic problem we need to solve is finding a parameter tuple $\bar v$, and for the model learning problem a formula $\phi(\bar x\smid\bar y)$, such that the hypothesis $\llbracket\phi(\bar x\smid\bar v)\rrbracket^B$ is consistent with, or minimizes the error on the training examples.
 
The input of the learning algorithms (both for parameter and model learning) consists of a set $T=\{(\bar u_1,\lambda_1),\ldots,(\bar u_t,\lambda_t)\}$ of labeled examples, but the algorithms also need access to the background structure $B$. We usually think of $B$ as being very large, and we want to avoid holding it in main memory or even looking at the whole structure. That is, we are looking for learning algorithms with a running time that is polynomial in $t$ (the number of training examples), but sublinear in the background structure $B$ under a reasonable model of accessing $B$. In \citep{grorit17}, the background structure $B$ is a graph, presumably of small degree, and the learning algorithms only have \emph{local access} to $B$, that is, they can only retrieve the neighbors of vertices that they already hold in memory. Initially, these are the vertices appearing in the training examples. 
The main result of \citep{grorit17} is that model learning for first-order logic is possible in time polynomial in the number $t$ of training examples and the maximum degree $d$ of the background graph $B$. 
The strings that we study as background structures in this paper are equipped with the $\leq$-relation which is of unbounded degree. Hence the results of \citep{grorit17} do not apply in this setting.
The polynomial that bounds the running time depends on the quantifier rank $q$ of the formula  $\phi(\bar x\smid\bar y)$ and the lengths $k,\ell$ of the tuples $\bar x,\bar y$. The crucial point is that this running time is independent of the size $n$ of the background structure (in a uniform cost model; otherwise it is poly-logarithmic in $n$).

\subsection{Our Results}
For the strings studied as background structures in this paper, we have also have a natural notion of local access: algorithms are only allowed to (directly) access the successor and predecessors of positions of a string that they already hold in memory. Our first result (Theorem~\ref{thm:noConsistentSublinearAlgorithmOnWords}) is negative: we prove that every (model or parameter) learning algorithm producing an FO-definable hypothesis consistent with the training examples (if there is one) necessarily needs time at least linear in $n$. Only if $\phi(\bar x\smid\bar y)$ is quantifier-free (Theorem~\ref{thm:quantifierFreeSingleInput}) or existential with only one instance variable, that is, $k=1$, (Theorem~\ref{thm:unary-existential}) we obtain a model learning algorithm for FO running in time polynomial in $t$, independently of $n$. We can strengthen our linear lower bound in such a way that it already applies to existential FO-formulas with two free instance variables (Theorem~\ref{thm:binaryExistentialFormualsNotSublinear}).

The negative results are not very surprising, because the local access model to the background string $B$ is extremely restrictive. For example, it is impossible for an algorithm to find the first position in $B$ that is labeled by symbol `$a$'. We also consider a less restrictive access model, where we allow a linear time pre-processing of the background string to build an index data structure that allows for more global access to $B$. The pre-processing takes place before the algorithm sees the training examples. Then in the actual learning phase, the algorithm only has local access to $B$ and the index structure, that is, is only allowed to follow pointers. Our main result (Theorem~\ref{the:precomputation-unary}) states that after such a linear time (in $n$) pre-processing phase, both parameter and model learning for MSO-definable hypotheses are possible in time polynomial in $t$. 

Technically, this theorem heavily relies on the connections between monadic second-order logic, finite automata, and semi-group theory. The index data structure we built in the pre-processing phase is the Simon Factorization Forest \citep{Simon90,Kufleitner08} for a suitable monoid associated with MSO-definable hypotheses.

\subsection{Related Work}
Closely related to our framework is that of inductive logic programming (ILP)
\cite[see, for example,][]{cohpag95,kiedze94,mug91,mug92,mugder94}. The two main differences are that we encode background knowledge in a background structure, whereas the ILP framework axiomatizes it in a background theory, and that we work with MSO, whereas ILP focuses on FO, possibly in a recursive setting. Other recent logical frameworks for machine learning, mainly in the context of database and verification applications can be found in \citep{aboangpap+13,bonciusta16,lodmadnei16,garneimadrot16,jorkai16}.

There are also numerous results on learning automata and regular languages, negative \citep{Angluin78,Gold78,PittW93,KearnsValiant94,Angluin90} as well as positive \citep{Angluin87,RivestS93,computationalLearningTheory,GarciaO92}, the latter mainly in an active learning framework. Technically, all these results seem unrelated to ours.

\section{Preliminaries}

We considers strings (or words) over an alphabet $\Sigma$. The set of all such finite strings is denoted by $\Sigma^*$, and the empty word by $\varepsilon$.
In the logical setting, we view words as structures $B$ over the signature $\tau = \{<, (R_a)_{a \in \Sigma})$ with universe $U(B) = \{1,\ldots,n\}$ for words of length $n$. We also refer to the elements of $U(B)$ as positions of the word. The relation $<$ is the natural ordering of the positions, and each $R_a$ is a unary predicate for the $a$-labeled positions. Note that we do not have a direct successor relation for positions (this is only relevant for the results in Section~\ref{sec:existential} which can be extended to include the successor relation but become much more technical in that setting).
In general, we do not distinguish between the relational representation and the sequence of alphabet symbols.

We use standard first-order logic (FO) over these word structures. 
Monadic second-order logic (MSO) extends FO by additional quantification over sets of positions.
We use lowercase letters $x,y,z$ to denote first-order variables and the corresponding uppercase letters for set variables.
The \emph{quantifier rank} of a formula $\varphi$ is the maximal number of nested quantifiers in the formula.

We refer to the introduction for the basic definitions on our learning model. For a formula $\varphi(\bar x \smid \bar y)$ we define the \emph{arity} of $\varphi(\bar x \smid \bar y)$ to be the number of instance variables in $\bar x$. 
For the case of a single instance variable $x$ we speak of \emph{unary} formulas. 

For a word $B$, a training set $T \subseteq U(B)^k \times \{0,1\}$ is called \emph{$\phi$-consistent} if there are parameters $\bar v \in U(B)^\ell$ such that for all $(\bar u,\lambda) \in T$ the classification is $\lambda = \llbracket\phi(\bar x\smid\bar v)\rrbracket^B(\bar u)$. We also say that $\bar v$ is consistent with $\varphi$, $B$,  and $T$. 

As mentioned in the introduction, we distinguish between parameter learning and model learning. For the parameter learning problem, we assume a fixed formula $\varphi(\bar x \smid \bar y)$. A \emph{parameter learner} $\mathcal{L}$ for $\varphi$ has as input a word $B$ and a training set $T$ over $B$. Using the access model for $B$ explained in the introduction, $\mathcal{L}$ produces as output a tuple $\bar{v} = \mathcal{L}(B,T)$ of parameters. We say that $\mathcal{L}$ is a \emph{consistent parameter learner} for $\varphi$ if $\mathcal{L}(B,T)$ is consistent with $\varphi$, $B$, $T$ for all possible $B$ and $\varphi$-consistent $T$.

For the model learning problem, the formula $\phi(\bar x\smid\bar y)$ has to be found by the algorithm. However, we do fix a maximum quantifier rank $q$ and the number $\ell=|\bar y|$ of parameters. Note that the arity $k=|\bar x|$ is also fixed, because the instance space of our learning problem is $U(B)^k$. The view that $q,k,\ell$ are fixed, which will be important for our complexity analysis, is motivated by an analogy with database theory. Our whole declarative learning framework is inspired by the declarative framework of relational database systems. The logic in which we specify our models and hypotheses corresponds to the database query language, and the background structure corresponds to the database. It is common in database theory to analyze the complexity of the query evaluation problem, which corresponds to the learning problems we consider here, by regarding the query as fixed and the database as the variable input; this perspective is known as \emph{data complexity} \citep{var82}. The data complexity approach is justified by arguing that queries are usually human-written and not too large, certainly in comparison with the size of the data, and that therefore we can treat the query size as constant. Similarly, if we aim for explanatory and human-understandable models in machine learning, and want to avoid over-fitting, we may want to restrict the quantifier rank and the number of the parameters of the formulas defining the models. Adapting the database terminology, we may say that in this paper we analyze the \emph{data complexity} of parameter and model learning. 

Now if we fix $q,k,\ell$, there is only a finite number of formulas $\phi(\bar x\smid\bar y)$ that our algorithm needs to consider, because up to logical equivalence there is a only a finite number of MSO-formulas of quantifier rank at most $q$ with at most $k+\ell$ free variables. We may actually assume that there is a fixed formula $\phi(\bar x\smid\bar y)$ of quantifier rank $q$ and with $|\bar y|=\ell$. Then the goal of our model learning algorithm is to compute, given a training set $T$ and structure $B$ such that $T$ is $\phi$-consistent, a formula $\phi'(\bar x\smid\bar y')$ and a parameter tuple $\bar v'$ such that $\bar v'$ is consistent with $\phi'$, $B$, and $T$. Furthermore, for simplicity we may assume that the learning algorithm ``knows'' $\phi(\bar x\smid\bar y)$. These assumptions are justified by the observation that otherwise the algorithm can iterate through all the (finitely many) possible formulas. Note that this makes model learning a simpler problem than parameter learning (in the sense that an algorithm for parameter learning yields a model learning algorithm).
In \citep[Section 3]{grorit17}, there is a simple example illustrating that model learning can be strictly simpler (also see Example~\ref{exa:1} below).  We may further allow the formula $\phi'(\bar x\smid\bar y')$ to have a larger quantifier rank $q'\ge q$ and a larger number $\ell':=|\bar y'|\ge\ell$ of parameter variables than $\phi$.  We refer to such a learning algorithm $\mathcal L$ as a \emph{$(q',\ell')$-formula learner}. We say that $\mathcal{L}$ is a \emph{formula learner} for $\varphi$ if it is a $(q',\ell')$-formula learner for $\varphi$ for some numbers $q',\ell'$. Note that each parameter learner for $\varphi$ is also a $(q,\ell)$-formula learner for $\varphi$. 

\begin{example}\label{exa:1}
  Suppose our background string $B$ is over the alphabet $\Sigma=\{a,b\}$. 

  Let $\phi_1(x\smid y)= R_a(x)\wedge x\le y$. Let $\mathcal L_1$ be the algorithm that, given a training set $T$ (and local access to $B$), returns the largest position $v$ such that $(v,1)\in T$. It is easy to see that $\mathcal L_1$ is a consistent parameter learner for $\phi_1$.

  Now consider $\phi_2(x\smid y)= R_a(x)\wedge R_b(y)\wedge x\le y$. Then a consistent parameter learner for $\phi_2$ has to search for the first position $v$ in $B$ that is labeled by $b$ and is greater than or equal to all $u$ such that $(u,1)\in T$. It may take time linear in $|B|$ to find such a $v$.

  However, it is easy to construct a $(0,1)$-formula learner $\mathcal L_2$ for 
  $\phi$: given $T$, it returns the formula $\phi_1(x\smid y)$ and as parameter the largest position $v$ such that $(v,1)\in T$.  
\end{example}

We only require our learning algorithms to return hypotheses consistent with the training set. In fact, all of our algorithms can be generalized in such a way that they return a hypothesis with minimum training error if there is no consistent one (we leave the details to the full version of this paper).

To justify that such learning algorithms return hypotheses that generalize well, we appeal to the standard result from PAC-learning (also see Section~\ref{sec:pac}): A consistent learner is also a PAC-learner over a hypothesis space with bounded VC-dimension using a training sequence of size polynomial in the VC-dimension and the constants of the error bounds  \citep[see][]{shaben14}.
The following theorem shows that our hypothesis space is of bounded VC-dimension.

\begin{theorem}[\citet{grohe2004learnability}]
  For $q,k,\ell$ there is a $d$ such that for every string $B$ the family of all hypotheses $\llbracket\phi(\bar x\smid\bar v)\rrbracket^B$, where $\phi(\bar x\smid\bar y)$ is an MSO-formula of quantifier rank at most $q$, $|\bar x|=k,|\bar y|=\ell$, and $\bar v\in U(B)^\ell$, has VC-dimension at most $d$.  \end{theorem}

\section{Non-Learnability of Unary FO-Definable Concepts}\label{sec:counterexample}

\begin{theorem}\label{thm:noConsistentSublinearAlgorithmOnWords}
 There is no consistent formula learner for unary FO formulas whose running time is sublinear in the length of the string.
\end{theorem}

The theorem follows immediately from the following lemma.

\begin{lemma} \mgrand{Rephrased the lemma for alter reference}
  \label{lem:counterexample} 
  There is an FO formula
  $\varphi(x\smid y) = \exists z \forall z' \psi(z,z',x,y)$ with
  quantifier-free $\psi$ such that for all 
  $(q,\ell)$-formula learners $\mathcal L$ with a sublinear running time the
  following holds.
  There is a string $B$ and a $\varphi$-consistent training set $T$ of
  size $|T|=2\ell+3$ such that the hypothesis  $H_T$ produced by
  $\mathcal L$ on input $B$ and $T$ is not consistent with $T$.
\end{lemma}

\begin{proof}\mgrand{Beweis etwas angepasst}
    We consider strings over the alphabet $\Sigma = \{a,b,c\}$.  
  We view strings over $\Sigma$ as consisting of \emph{$a$-blocks} (sequences of successive $a$-positions) that are separated by $b$ or $c$. 
  The \emph{entry} of an $a$-block is the position directly before that block, which is then labeled $b$ or $c$ (if the string starts with $a$, then the first block does not have an entry).
  The formula $\varphi(x\smid y)$ selects those $a$-positions whose block entry is either before $y$ and labeled $b$, or behind (including) $y$ and labeled $c$. 
  So the behavior of the formula switches at the parameter position. 
  The concrete formula is
\begin{align}
\phi(x\smid y) = R_a(x) \land \exists z (z<x) &\land ((R_b(z)\land z<y) \lor (R_c(z) \land z\geq y))\nonumber\\ 
											&\land \forall z' ((z<z'<x)\rightarrow R_a(z')) \label{fomula:phi}
\end{align}
Let $\mathcal{L}$ be a $(q,\ell)$-formula learner for arbitrary numbers $q,\ell$ whose running time is sublinear in the length of the string.  

We choose $s$ such that for every MSO formula with $\ell$ parameters
and quantifier rank at most $q$, there is an equivalent finite
automaton with at most $s$ states. Then we choose $r$ in such a way
that the running time of the learner $\mathcal L$ on an input string
$B$ of length 
\[
n=(2\ell+3)(3s!)(2r+2)
\]
is at most $r$. This choice of $r$ is possible since the runtime of
$\mathcal{L}$ is sublinear in the length of the input string.

Next, we construct strings $B_0,\dots,B_\ell$ of length $n$ and parameter positions $v_i$ such that $\varphi(x\smid y)$ selects the same set of positions for all $B_i$ with parameter $v_i$. 
These strings have the following shape: 
\[
B_i = (A_bA_c)^iA_b(A_bA_c)^{\ell+1-i},
\]
where the strings $A_b$ and $A_c$ are defined as $A_b = (ba^{2r+1})^{3s!}$ and $A_c = (ca^{2r+1})^{3s!}$. 

Each $B_i$ contains two successive $A_b$. 
The parameter position $v_i$ for $B_i$ is chosen to be the first $b$ in the second of the two successive $A_b$. 
Since the behavior of $\varphi$ flips at the parameter position, it should be clear that $\varphi$ indeed selects the same set of positions for each $B_i$ with parameter $v_i$, or formally, $\llbracket\phi(x,v_i)\rrbracket^{B_i}=\llbracket\phi(x,v_j)\rrbracket^{B_j}$ for all $i,j$.
This allows us to construct a single training set that is $\varphi$-consistent over all $B_i$. 
The training set $T$ contains one position in each substring $A_b$ or $A_c$ of $B_i$. 
The positions are chosen in the middle of an $a$-block (of length $2r+1$), which itself is in the middle of all the $a$-blocks of the respective substring $A_b$ or $A_c$. 
Formally, these are the positions $|A_b| \cdot j +\frac{|A_b|}{2} + r +2$ for $j \in \{0,\ldots,2\ell+2\}$ (note that $|A_b| = |A_c|$ so we do not need to distinguish them in the definition of the  positions). 
The classification of these positions starts with $1$ and then alternates between $0$ and $1$ (identical for every $B_i$).

By the choice of $r$, the algorithm $\mathcal{L}$ only has access to $a$-positions when it is executed on $B_i$ with training set $T$. 
Hence, it does not see any difference between the $B_i$ and produces the same hypothesis for all $B_i$.
In particular this includes the same set of at most $\ell$ parameters since these are part of the hypothesis. 
So there is one string $B^*$ for which no parameter is inside the two successive $A_b$ substrings. 
We now show that the hypothesis on this $B^*$ cannot distinguish the two positions in the two successive $A_b$ substrings, whereas their classification in the training set is different. 
This shows that the hypothesis produced by $\mathcal L$ on input $B^*$ and $T$ is not consistent with $T$.
The hypothesis cannot distinguish the examples from the $A_bA_b$ block.
To show this we observe that the formula learner $\mathcal L$ produces a hypothesis $\psi(x,\bar{v}')$ for $B_i$ and $T$ (we use $\bar{v}'$ for
the parameters because the notation $v_i$ is already in use in this proof).

For the hypothesis formula $\psi(x,\bar y)$ there is an equivalent DFA (deterministic finite automaton) $\mathcal A$. This DFA reads words over $\Sigma$ that are annotated in some form to mark the position $u$ of the instance variable, and the positions of the parameters. In particular, it accepts $B_i$ annotated with a position $u$ and the positions $\bar{v}'$ from the hypothesis if, and only if, $\psi(u,\bar{v}')$ holds in $B_i$. By the choice of $s$, there is such a DFA $\mathcal A$ with at most $s$ states.

We show that this DFA wrongly classifies one of the examples from the two successive $A_b$ subwords in $B_i$. More formally, let $u_1 < u_2$ be the positions from the training set that fall into the two successive $A_b$ subwords of $B_i$. Let $B_i^1$ and $B_i^2$ be the strings annotated with the parameters $\bar{v}'$ and additionally $B_i^1$ is annotated with $u_1$ as instance for $x$, and similarly for $B_i^2$ and $u_2$. We know that these two positions are classified differently in the training set. Note also that $B_i$ was chosen such that no parameter is inside the two successive $A_b$'s.

We analyze the runs of $\mathcal{A}$ on $B_i^1$ and $B_i^2$ on the two successive $A_b$ subwords.  The idea is illustrated in Figure~\ref{fig:stateTransitionsCounterexample}. Let $A'$ denote $A$ with the middle $a$-position carrying the marker for the  position of $x$, and let $A^* = A^{\frac{1}{2}s!}A'A^{\frac{1}{2}s!-1}$. Then the two successive $A_b$ subwords in $B_i^1$ and $B_i^2$ including the markers for $u_1$ and $u_2$ are of the form $C_1 = A^{s!}A^*A^{4s!}$ and $C_2 = A^{4s!}A^*A^{s!}$. This implies that the state before $A^*$ is the same in both runs: After reading the first $s!$ copies of $A$, the state is some $\sigma_1$ for both words (up to this point, $B_i^1$ and $B_i^2$ are the same). Then $B_i^1$ is followed by $A^*$. In $B_i^2$, the DFA reads another $3s!$ copies of $A$. Since $s!$ is a multiple of every possible length of a loop in $\mathcal{A}$, and $\mathcal{A}$ has already read $s!$ copies of $A$, the state before $A^*$ in $B_i^2$ is also $\sigma_1$. The same argument is used to show that both runs end on the same state $\sigma_2$ after having read $C_1$ and $C_2$, respectively.

Therefore, $\mathcal{A}$ cannot distinguish $C_1$ and $C_2$ and thus accepts both, $B_i^1$ and $B_i^2$, or rejects both.
This means that $\mathcal A$ will either accept or reject both training examples $u_1$ and $u_2$.
\end{proof}

\begin{figure}[h!]
\begin{center}
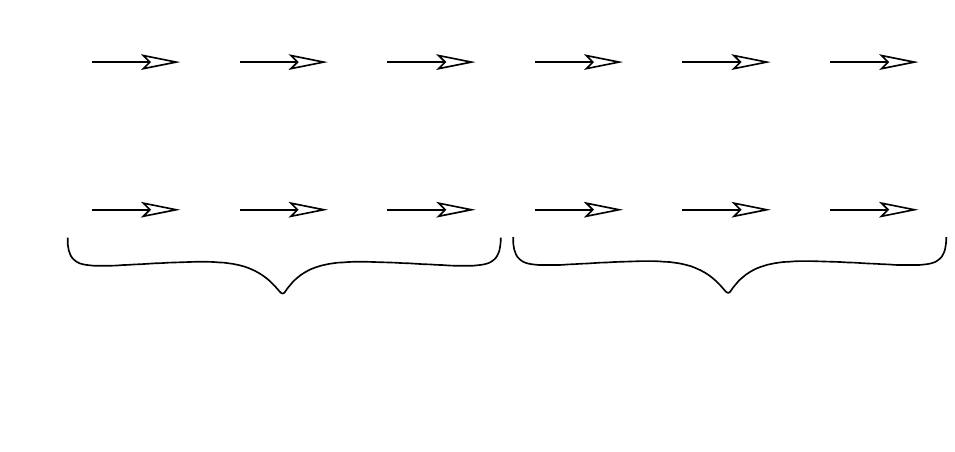
\end{center}
\caption{State transitions of $\mathcal A$ on the $A_bA_b$ block for the two different training examples}
\label{fig:stateTransitionsCounterexample}
\end{figure}

\subsection{PAC Learning}
\label{sec:pac}
The general aim in machine learning is to perform well on unseen examples.
This is formalized in Valiant's probably approximately correct learning model. 
The idea is that a learning algorithm is \emph{probably approximately correct} (PAC), if over most of the training sets ('probably') it outputs a classification algorithm (or hypothesis) which has a low expected error on new examples ('approximately').
To get any bounds on the expected error, a fixed underlying but unknown probability distribution $\mathcal D$ over the examples is assumed. 
Then the examples are chosen independently from this distribution $\mathcal D$.
When talking about training sets, we implicitly assume that those are chosen independently according to some unknown but fixed distribution $\mathcal D$.

Here we show that there is no PAC learning algorithm with sublinear runtime.
\begin{definition}
Let $\mathcal L$ be a learning algorithm which outputs on input of $B$ and $T$ a hypothesis $H_T$.
Then $\mathcal L$ is \emph{probably approximately correct (PAC)} if for all probability distribution $\mathcal D$ over the instance space\mgrand{Added ``for all distributions''}
 \[\Pr_{T\sim \mathcal D} \left[ \err_\mathcal D(H_T)<\epsilon \right] > 1-\delta.\]
\end{definition}

Here the probability is taken over the training set $T$ where $T\sim \mathcal D$ means that the training examples are chosen independently according to $\mathcal D$. Furthermore,
$\err_\mathcal D(H_T)$ is the expected error on (new) examples chosen according to $\mathcal D$.

This means that a learning algorithm is \emph{probably approximately correct (PAC)} if with high probability over the choice of the training set $T$, the expected error of the hypothesis $H_T$ on new instances is low.

\begin{theorem}
Let $\mathcal L$ be a sublinear $(q,\ell)$-formula learner.
Then $\mathcal L$ is not a PAC learning algorithm.
\end{theorem}

\begin{proof} \mgrand{Den Beweis vorher habe ich nicht so recht
    verstanden, zumindest nicht diese komische Argument wo man erst
    einen der Strings zufällig auswählen muss. Ich glaube, so stimmt
    es (und ist einfacher), ich hoffe, ich übersehe da nichts.}
  We choose a string $B$ and a training set $T$ of size $|T|=2\ell+3$
  according to Lemma~\ref{lem:counterexample}. Now we let $\mathcal D$
  be the uniform distribution over the position appearing in $T$. Then
  if we draw examples from the distribution $\mathcal D$, the learner
  only sees examples from $T$, and hence we know that it makes a
  mistake. There is a small technical issue here: if we draw the
  examples randomly from $T$, the learner $\mathcal L$ may actually
  only see a subset $T'\subseteq T$, because some may be
  repeated. However, without loss of generality we may assume that
  $\mathcal L$ does not perform worse if it sees
  more examples. 
  \mrrand{We further assume that the quality of the returned hypothesis does not depend on possible repetitions and is independent of the order of the examples.}
  Hence if we denote the hypothesis produced by
  $\mathcal L$ on input $B$ and $T'\subseteq T$ by $H_{T'}$, then
  $H_{T'}$ is not consistent with $T$, and we have
  \[
  \err_{\mathcal D}(H_{T'})\ge\frac{1}{2\ell+3},
  \]
  because $H_{T'}$ is wrong on at least one of the $2\ell+3$ elements
  in the support of $\mathcal D$. Thus with
  $\epsilon=\frac{1}{2\ell+3}$ we have
  \[
  \Pr_{T'\sim\mathcal D}[\err_{\mathcal D}(H_{T'})<\epsilon]=0.
  \]
  This implies that $\mathcal L$ is not a PAC learning algorithm.
\end{proof}

\section{Quantifier-Free and Existential Formulas}\label{sec:existential}
The non-learnability result from Section~\ref{sec:counterexample} applies to formulas with at least one quantifier alternation (Lemma~\ref{lem:counterexample}). 
In this section we therefore consider simpler classes, namely quantifier-free and existential formulas.
A quantifier-free formula $\psi$ consists of atoms checking membership in the relations $\tau = \{ <, (R_a)_{a\in \Sigma} \}$ or boolean combinations of those.
Existential formulas are of the form $\phi(\bar x) = \exists \bar z \psi(\bar x,\bar z)$ where $\psi$ is quantifier-free.

For quantifier-free and unary existential formulas, there are formula learners running in time polynomial in $|T|$. 

\begin{theorem}\label{thm:quantifierFreeSingleInput}
  There is a  consistent formula learner for unary quantifier-free formulas whose running time is in $\mathcal O(|T|)$ for a training set $T$.
  For arbitrary quantifier-free formulas with $k$ instance and $\ell$ parameter variables there is a consistent formula learner running in time $\mathcal O((2|T|k+1)^{\ell}|T|)$.
\end{theorem}

Note that even in the more general case, the runtime of the algorithm is polynomial in $|T|$ for a fixed $\ell$.
The proof uses the fact that evaluating a quantifier-free formula over an ordered word only depends on the labels and relative positions of the nodes assigned to the free variables. 
The number of possible different configurations therefore only depends on $k$ and $\ell$ which means that all of those can be checked.
For a proof of this theorem we refer to the full version of this article.

In Theorem \ref{thm:unary-existential} we generalize the first part of Theorem~\ref{thm:quantifierFreeSingleInput} to unary existential formulas. 
In contrast to the quantifier-free case, where $\ell$ can be taken from $\phi$ for some $\phi$-consistent training set, the constructed hypothesis uses a relatively long formula.
The formula constructed for the hypothesis is based on the observation that existentially quantified conjunctions can only define an interval per label $a\in\Sigma$.
Theorem \ref{thm:binaryExistentialFormualsNotSublinear} states that this cannot be extended to arbitrary existential formulas. Again, we refer to the full version of the paper for proofs of these theorems.\mgrand{Beweise auch raus und auf full version verwiesen}

\begin{theorem}\label{thm:unary-existential}
  There is a  consistent formula learner for unary existential formulas whose running time is in $\mathcal O(|T|)$ for a training set $T$.
\end{theorem}

\begin{theorem}\label{thm:binaryExistentialFormualsNotSublinear}
There is no consistent formula learner for binary existential formulas with one parameter that runs in sublinear time.
\end{theorem}

\newcommand{\blank}{?}
\newcommand{\tree}{\mathcal{T}}
\newcommand{\Mphi}{\hat{M}}
\newcommand{\Fphi}{\hat{F}}
\newcommand{\hphi}{\hat{h}}

\section{Indexing} \label{sec:preprocessing}

While it is impossible to learn the parameters of a fixed formula in sublinear time, we now consider the case that the learning algorithm can preprocess and index the underlying string before it enters the learning phase for this string.
In this setting, the learning algorithm consists of two phases, starting with a linear time (in $|B|$) indexing phase, in which the algorithm can add an auxiliary data structure to $B$.
The learning phase (sublinear in $|B|$) then can use this auxiliary data structure to compute consistent parameters $\bar v$ for a given training set $T$.
We refer to the running times of these two phases as \emph{indexing time} and \emph{learning time}, respectively.

The main result in this section is about the learnability of unary MSO formulas in the indexing model, as stated in the next theorem. Most of this section is devoted to the proof of this theorem. At the end of the section we briefly mention the case of MSO formulas with higher arity.
\begin{theorem} \label{the:precomputation-unary}
There is a consistent parameter learner for unary MSO formulas with indexing time $O(|B|)$ for a string $B$ and learning time $\mathcal O(|T|)$ for a training set $T$ over $B$. 
\end{theorem}

As an example, consider the formula used in the proof of Lemma~\ref{lem:counterexample}. The indexing phase could annotate each position with the information whether it is in an $a$-block preceded by $b$ or by $c$. With this additional information it is easy to find a consistent parameter setting for a given training set.

For this specific example, it is sufficient to simply annotate each position in $B$ with some extra information. In the general case, the auxiliary data structure is a tree whose leafs are the positions of $B$, and each node contains information about the substring of $B$ in the subtree below this node. More formally, we use a factorization tree of the string w.r.t.~some finite monoid. We start by introducing monoids and factorization trees, and then define the specific monoid that we use in the learning algorithm. We refer the reader to \cite{Colcombet11} and \cite{Bojanczyk12} for some recent expositions that explain the connections between monoids and regular languages in more detail.

A \emph{monoid} $(M,\cdot,1_M)$ consists of a set $M$, an associative multiplication operation $\cdot$ on $M$, and a neutral element $1_M$ for this operation. Often we simply write $M$ for the monoid, and write the multiplication of two elements $m_1$ and $m_2$ as $m_1m_2$. The set of all finite words over an alphabet $\Sigma$ with concatenation as multiplication and the empty word as neutral element is called the free monoid (generated by $\Sigma$). A mapping $h:M\rightarrow M'$ for monoids $M$ and $M'$ is called a monoid morphism (just morphism for short) if 
$h(m_1m_2) = h(m_1)h(m_2)$ for all $m_1,m_2 \in M$, and $h(1_M) = 1_{M'}$.

For a finite monoid $M$, a morphism $h: \Sigma^* \rightarrow M$, and a subset $F \subseteq M$ of accepting monoid elements, we define the language $L(M,h,F) = \{A \in \Sigma^* \mid h(A) \in F\}$. A well known theorem implicitly mentioned in early papers of automata theory by \cite{RabinS59} states that a language is regular if, and only if, it can be accepted by a finite monoid in this way.%

We now turn to factorization trees. These can be seen as index structures for finite words. 
Our techniques are based on the same ideas as the ones described in \cite{Bojanczyk12}.

Let $M$ be a finite monoid, and let $s = m_1,m_2,\ldots,m_n$ be a sequence of elements from $M$. A \emph{factorization tree} $\tree$ of $s$ is a finite ordered tree (the successors of a node are ordered) whose nodes $v$ are labeled by elements $\tree(v)$ of $M$, such that
\begin{itemize}\itemsep-2pt
\item the sequence of leaf labels is $s$,
\item for each inner node $v$ with children $v_1, \ldots, v_i$, the label of $v$ is the product of the monoid elements at its children: $\tree(v) = \tree(v_1) \cdot \tree(v_2) \cdots \tree(v_i)$
\end{itemize}
Note that each node $v$ of $\tree$ defines an infix (factor) $m_j,...,m_{j'}$ of $s$ corresponding to the leafs in the subtree below $v$. The label $\tree(v)$ of $v$ is the product $m_j...m_{j'}$ of these elements.

For example, one can use a binary tree, whose height is then logarithmic in the length of $s$. We use a class of factorizations introduced by \cite{Simon90} that also can have nodes of higher arity with a specific property. 

A \emph{Simon factorization tree} $\tree$ of $s$ is a factorization tree with the following additional property:
\begin{itemize}\itemsep-2pt
\item if a node $v$ of $\tree$ has more than two children $v_1, \ldots, v_i$, then the labels of all the children are the same, and this label $e$ is an \emph{idempotent} element of $M$, that is $ee = e$ (it follows that $\tree(v) = e$, too).
\end{itemize}
We refer to such nodes as \emph{idempotent nodes}.

The following theorem is due to \cite{Simon90}. For the bound of $3|M|$ see \cite{Colcombet11,Kufleitner08}.
\begin{theorem}[Simon factorization theorem]\label{thm:simonFactorization}
For every sequence $s = [m_1,m_2,\dots,m_n]$ of monoid elements from $M$, there is a factorization tree of height at most $3|M|$.
This factorization tree can be computed in time $\text{poly}(|M|)\cdot n$.
\end{theorem}

Factorization trees can also be applied for strings over an alphabet $\Sigma$ (instead of a sequence of monoid elements), given a monoid morphism $h: \Sigma^* \rightarrow M$. A \emph{Simon $h$-factorization tree} for a string $B = a_1 \cdots a_n \in \Sigma^*$ is a Simon factorization tree for the sequence $[h(a_1),\ldots,h(a_n)]$.

We now turn to the monoid that we use for building a factorization tree in the indexing phase of the learning algorithm. We actually define two monoids, where the second one is used for the factorization. Its elements consist of sets of elements of the first monoid that we define. 

In the following, let $\varphi(x\smid\bar{y})$ be a unary MSO formula with parameter variables $y_1, \ldots, y_\ell$. The formula $\varphi(x\smid\bar{y})$ naturally defines a set of strings $\hat{L}(\varphi)$ over the alphabet $\hat \Sigma = \Sigma \times 2^{\{y_1,\dots, y_\ell \}} \times  \{ \blank,0,1 \}$. The first component of a string $\hat B \in \hat{\Sigma}^*$ defines a string $B$ over $\Sigma$. The third component encodes a training set $T_{\hat{B}}$, where $\blank$ indicates that the position is not in $T_{\hat{B}}$, and $0,1$ correspond to the classification of the position in $T_{\hat{B}}$. The second component is supposed to encode the parameter setting. We say that a string $\hat B\in\hat\Sigma^*$ \emph{contains} $y_i$ if there is a position $v$ labeled by $(a,Y,b)\in\hat \Sigma$ such that $y_i\in Y$, and we say that that  $\hat B$ contains $y_i$ \emph{exactly once} if there is exactly one such position. If $\hat B$ contains each $y_i$ exactly once then the second component encodes a valid parameter setting $\bar v=(v_1,\ldots,v_\ell)$.
\mgrand{re-phrased}

Then we let $\hat{L}(\varphi)$ be the set of all strings $\hat B\in\hat\Sigma^*$ such that $\hat B$ contains each $y_i$ exactly once, yielding a parameter setting $\bar v$, and if $B\in\Sigma^*$ is the projection of $\hat B$ to the first component and  $T_{\hat{B}}$ is the training set encoded by the third component, then $\llbracket\phi(x\smid\bar v)\rrbracket^B$ is consistent with $T_{\hat{B}}$.
\mgrand{re-phrased}

It is not difficult to see that a finite automaton for the formula $\varphi(x\smid\bar{y})$ can be modified to obtain a finite automaton for the language $\hat{L}(\varphi)$, which means that $\hat{L}(\varphi)$ can also be accepted by a finite monoid.
Let $\Mphi$ be a finite monoid, $\hphi: \hat{\Sigma} \rightarrow \Mphi$ be a monoid morphism, and $\Fphi \subseteq \Mphi$ such that $(\Mphi,\hphi,\Fphi)$ accepts $\hat{L}(\varphi)$.

\mgrand{moved here from the proof of Lemma 11}
For every subset $K\in2^{\{y_1,\dots,y_\ell\}}$ of the parameters, we let
\[
\Mphi_K = \{\hat h(A)
 \mid  A \in \hat{\Sigma}^* \text{ contains all }y_i \in K \text{ exactly once but does not contain any } y_i \not\in K\},
\]
and we let
\[
\Mphi_\bot = \{\hat h(A)
 \mid  A \in \hat{\Sigma}^* \text{ contains some $y_i$ more than once}\}.
\]
Without loss of generality we may assume that the sets $\Mphi_K$ for $K \in2^{\{y_1,\dots,y_\ell\}}\times\{\bot\}$ are mutually disjoint. It is easy to see this, the idea is that we can introduce copies $m_K$ of all elements $m$ and adjust the homomorphism $\hat h$ accordingly. 

Then $\Mphi$ is the disjoint union of the sets $\Mphi_K$ for $K\in2^{\{y_1,\dots,y_\ell\}}\cup\{\bot\}$. Observe that $\Fphi \subseteq \Mphi_{\{y_1,\dots,y_\ell\}}$ and that no substring of a string in $\Fphi$ is contained in  $\Mphi_\bot$.

We now define a second monoid $\mathcal M$, which is used for the factorizations.
The monoid $\Mphi$ contains information about the parameters as encoded in the strings. In the learning setting, these parameters are unknown and we need to synthesize parameters consistent with the training set. We therefore introduce a monoid that contains information about all possible parameter settings that could be encoded in the strings.

For this purpose, let $\Gamma = \Sigma \times \{ \blank ,0,1 \}$ be the alphabet without the component for the parameters. A string over $\Gamma$ encodes a string over $\Sigma$ together with a training set over $B$.

Let $f:\hat\Sigma^* \rightarrow \Gamma^*$ be the function that projects strings over $\hat{\Sigma}$ to the corresponding strings over $\Gamma$, removing the parameter component of $\hat \Sigma$.
Based on this projection, each string $A$ over $\Gamma$ defines a set $h(A)$ of elements of $\Mphi$ by
\[
h(A) =  \{ m\in \Mphi \mid \exists \hat{B}\in \hat\Sigma^* \mbox{ with } f(\hat{B}) = A \mbox{ and } \hphi(A) = m\}. 
\]
This defines a morphism $h: \Gamma^* \rightarrow \mathcal M$ using a new monoid structure $\mathcal M = \{ S ~|~ S\subseteq \Mphi \}$ with neutral element $1_\mathcal M = \{1_{\Mphi}\}$, the set containing only the neutral element of $\Mphi$, and multiplication $S_1 \cdot S_2 = \{ m_1 \cdot m_2 ~|~ m_1 \in S_1, m_2 \in S_2\}$.

In the following, we denote by $B_T \in \Gamma^*$ the string that encodes $B \in \Sigma^*$ with training set $T$ over $B$. In particular, $B_\emptyset$ denotes this string for the empty training set (so $B_\emptyset$ is $B$ extended with $\blank$ at every position).

The next lemma states that we can compute parameters that are consistent with a given training set $T$ over a string $B$, based on a Simon $h$-factorization of $B_T$.

\begin{lemma} \label{lem:compute-parameters}
Let $B \in \Sigma^*$ and let $T$ be a $\varphi$-consistent training set over $B$. Given a Simon $h$-factorization tree $\tree$ of $B_T$, one can compute in linear time in the height of $\tree$ a set of parameters over $B$ such that $T$ is consistent with the parameters.
\end{lemma}

\begin{proof}
A procedure for computing a consistent parameter setting as claimed in Lemma~\ref{lem:compute-parameters} is shown as
Algorithm~\ref{alg:traversalBinaryTreeUnaryRelation}. The idea and notations are explained below.

\begin{algorithm}[ht]
\LinesNumbered
\KwIn{Simon factorization tree $\tree$}
\KwOut{A consistent parameter setting $(y_1,\dots,y_\ell)$}
 $v \gets \text{root}(\tree)$

 Pick $m_{\textup{root}} \in \tree(v) \cap \Fphi$ 
 \tcp*[f]{an accepting element in the label of $v$}

 $\text{push}(m_{\textup{root}},v)$\;

\While {not empty stack}{
	 $(m,v) \gets \text{pop}()$\;

	 \eIf {$\text{leaf}(v)$}{ \tcp*[f] {set the parameters traced to this leaf}

         Let $K \subseteq \{y_1, \ldots, y_\ell\}$ be such that $m \in \Mphi_K$\;

         Set $y_i \gets v$ for each $y_i \in K$\;
     }
        {\tcp*[f]{descend further down the tree}

        Let $v_1$ be the first and $v_2$ be the last child of $v$\;
      
        \If {$v$ has two children}{
	            Pick $m_1 \in \tree(v_1)$, $m_2 \in \tree(v_2)$ with $m = m_1 m_2$\;
        }
        \If {$v$ has more than two children}{
            Let $e$ be the unique idempotent element in $\Mphi_\emptyset\cap\tree(v)$\;
	    
            Pick $m_1,m_2 \in \tree(v)$ with $m = m_1 e m_2$\;

\tcp*[f]{See Claim in the proof of Lemma~\ref{lem:compute-parameters}}
        }
            \lIf {$m_1 \not\in M_\emptyset$} {$\text{push}(m_1,v_1)$}
		\lIf {$m_2 \not\in M_\emptyset$} {$\text{push}(m_2,v_2)$}
	}
}
\Return{$\hat y = (y_1,\dots,y_\ell )$}\;

\caption{Computing a consistent parameter setting from a factorization tree}\label{alg:traversalBinaryTreeUnaryRelation}
\end{algorithm}
The variable $v$ is used for nodes of $\tree$, and $m$ for monoid elements of $\Mphi$. Recall that the nodes of $\tree$ are labeled with elements from $\mathcal{M}$, which are sets of elements of $\Mphi$.

The algorithm starts in the root of $\tree$, and picks some accepting element $m_{\textup{root}}$ in the label of the root. Such an accepting element exists, since we assume that $T$ is a $\varphi$-consistent training set. Thus, there is a parameter setting that is consistent with $T$. Adding this parameter setting to $B_T$ yields a string $\hat{B}_T \in \hat{\Sigma}^*$. Then $m_{\textup{root}} = \hphi(\hat{B}_T)$ is accepting and it is contained in the label of the root of $\tree$.

The algorithm then descends down the tree to find parameter positions that generate the accepting element chosen at the root.
It uses a stack because it has to descend on several paths (to find a position for each parameter).

\mgrand{Hier angepasst, weil einige Definition schon oben stehen}
Recall that $(\Mphi,\hphi,\Fphi)$ accepts $\hat{L}(\varphi)$, that
$\Mphi$ is the disjoint union of the sets $\Mphi_\bot$ and $\Mphi_K$ for $K\in2^{\{y_1,\dots,y_\ell\}}$, that the set $\Fphi$ of accepting elements is contained in $\Mphi_{\{y_1,\ldots,y_\ell\}}$, and that no string in $\Fphi$ has a substring in $\Mphi_\bot$. As all elements $m\in\Mphi$ the algorithm visits (and pushes to the stack in lines 3 and 18) are substrings of $m_{\textup{root}}\in\Fphi$, no such $m$ is an element of $\Mphi_\bot$.

In the main loop, Algorithm~\ref{alg:traversalBinaryTreeUnaryRelation} traces monoid elements that are not in $\Mphi_\emptyset$ further down the tree. Note that the elements $m\in \Mphi_\emptyset$ correspond to words that do not contain a parameter.

The algorithm pops the next pair $(m,v)$ with $m \in \Mphi$ and $v$ a node of $\tree$ from the stack. If $v$ is a leaf, then there is a set $K$ such that $M \in \Mphi_K$ and $K \not= \emptyset$ because elements from $M_\emptyset$ are never pushed onto the stack. The leaf $v$ corresponds to a position in the string $B$. This position is the value for the parameters in $K$.

If $v$ is an inner node, then $m$ can be written as product of $\Mphi$ elements in the labels at the children of $v$. If $v$ has only two children $v_1$ and $v_2$, then the algorithm can simply pick elements $m_1$ and $m_2$ in the labels of $v_1$ and $v_2$ whose product is $m$. If $v$ has more than two children, then it is an idempotent node. The choices made by the algorithm in this case are based on the following claim.
Intuitively, this claim shows that for finding a consistent parameter setting, it is sufficient to consider the first and the last child of idempotent nodes.

\bigskip\noindent
\textit{Claim:}
Let $S \subset \Mphi$ be the label of an idempotent node. Then $S$ contains a unique idempotent element $e \in M_\emptyset$, and each element $m$ of $S$ can be written as a product $m = m_1 e m_2$ with $m_1,m_2 \in S$.\mgrand{Es gibt auch idempotente Elemente in $\Mphi_{\bot}$. Deswegen habe ich die Formulierung hier, den Beweis, und auch Zeile 15 des Algorithmus etwas angepasst}

\medskip\noindent
\textit{Proof:}
  Let $v$ be an idempotent node with label $S$. 
  Since $\tree$ is a factorization tree of $B_T$, the node $v$ corresponds to a substring $A$ of $B_T$, and $S = h(A)$. 
  From the definition of $h(A)$ it follows that $S$ contains exactly one element $e \in \Mphi_\emptyset$ (the element for the empty parameter annotation of $A$). 
  The product of two elements from $\Mphi_\emptyset$ is also in $\Mphi_\emptyset$. 
  Since $S$ is idempotent, it follows that $ee \in \Mphi_\emptyset \cap S$, and thus $ee = e$. 

  We now show that each $m \in S$ can be written as $m = m_1 e m_2$ with $m_1,m_2 \in S$. Let $K \subseteq \{y_1, \ldots, y_\ell\}$ be such that $m \in \Mphi_K$. Prove this by induction on the size of $K$.

  If $K = \emptyset$, then $m = e = eee$, as shown above, and we let $m_1 = m_2 = e$. 
  Otherwise, since $S$ is idempotent, $m = m_1'm_2'$ with $m_1',m_2' \in S$. 
  From the definition of $\Mphi_K$ we obtain that $m_1' \in \Mphi_{K_1}$, $m_2' \in \Mphi_{K_2}$ with $K_1 \cup K_2 = K$ and $K_1 \cap K_2 = \emptyset$. 
  If $K_1 = \emptyset$, then we choose $m_1 = e$ and $m_2 = m_2'$ and obtain $m_1em_2 = eem_2' = em_2' = m_1'm_2' = m$. The case $K_2 = \emptyset$ is analogous. 
  If $K_1$ and $K_2$ are nonempty, then both are strict subsets of $K$. 
  By induction $m_1' = m_1'' e m_2''$ for $m_1'',m_2'' \in S$, and hence $m = m_1'' e m_2''m_2'$ and we can choose $m_1 = m_1''$ and $m_2 = m_2''m_2'$. 
  Since $S$ is idempotent, $m_2 \in S$. This completes the proof of the claim.
\hfill$\lrcorner$

\bigskip
For the correctness of the algorithm, one can prove that the parameters selected by the algorithm generate the accepting monoid element chosen at the root of $\tree$ (by an induction on the height of the node $v$ in the tree). This means that the choice of parameters is consistent with the training set.
The running time is linear in the height of $\tree$ because for each parameter there is at most one monoid element on the stack, which means that the algorithm follows at most $\ell$ paths in the tree.

This completes the proof of Lemma~\ref{lem:compute-parameters}.
\end{proof}

In order to apply Lemma~\ref{lem:compute-parameters} in our algorithm, we first have to compute a Simon $h$-factorization tree $\tree$ of $B_T$ for a given training set $T$. We can do this starting from a factorization of $B_\emptyset$, as stated in the following lemma.

\begin{lemma} \label{lem:compute-training-factorization}
Let $B \in \Sigma^*$ and let $T$ be a training set over $B$. From a Simon $h$-factorization tree $\tree_B$ of $B_\emptyset$, one can compute a Simon $h$-factorization tree $\tree$ of $B_T$ in time $\mathcal O(\mathit{height}(\tree_B)\cdot|T|)$. The height of $\tree$ is in $\mathcal O(\mathit{height}(\tree_B))$.
\end{lemma}
\begin{proof}
  The rough idea is as follows: We cut $B$ into factors at the positions occurring in $T$. Each position in $T$ becomes one factor (consisting of a single position), and the other factors are the substrings between these positions.
  Then we insert the modified monoid elements at the positions from $T$, by changing the $\blank$ into the classification of the position in $T$.
  For the longer substrings (that are not touched by the training set) we compute a factorization tree, which can easily be done based on the one for the whole string. We obtain one monoid element for each factor (at the root of the trees for the longer strings).
  For this new sequence of monoid elements we apply Theorem~\ref{thm:simonFactorization}, obtaining a Simon factorization, which can be combined with the existing factorization trees for the substrings to obtain a factorization tree for $B_T$. 

    More formally, let $u_1, \ldots, u_t \in \{1, \ldots, |B|\}$
\mgrand{Changed to $j_i$ to $u_i$, because we usually denote instances in the training set by $u$} be the positions of $B$ occurring in the training set $T$ in ascending order. Let $u_0 = 0$ and $u_{t+1} = |B|+1$ to simplify the following definitions.
  For $i \in \{0, \ldots,   t\}$ define $B_i$ to be the substring of $B_\emptyset$ from position $u_i+1$ to $u_{i+1}-1$, and for $i \in \{1, \ldots,   t\}$ let $\gamma_i = (a,c) \in \Gamma$ be the letter $a$ of $B$ at position $u_i$, and $c$ the classification of $u_i$ in $T$.
Then $B_0 \gamma_1 B_1 \gamma_2 \cdots B_{t-1} \gamma_t B_t = B_T$

For the substrings $B_i$ we can compute a Simon $h$-factorization tree $\tree_i$ from $\tree_B$. 
This is done in a similar fashion as described in \cite{Bojanczyk12} for evaluating queries for substrings on a Simon factorization tree. 
The idea is illustrated in Figure~\ref{fig:simon-tree}, which shows a Simon factorization tree for the sequence $[m_1,m_2,e,e,e,m_3,m_4,e,m_5]$ on the left-hand side of the figure. 
The right-hand side of the figure shows a  Simon factorization tree for the sub-sequence $[m_2,e,e,e,m_3]$. 
Basically one has to trace the paths towards the root from the left-most and right-most leaf nodes of the tree for the sub-sequence. 
Along these paths one has to update some labels, delete some nodes, and inserting some new nodes in order to maintain the structure of a Simon factorization tree. 
In Figure~\ref{fig:simon-tree}, the nodes labeled with a product of elements are the ones that have been inserted. 
One observes that the insertion of new nodes might increase the height of the tree, but the height can at most double. 

\begin{figure}
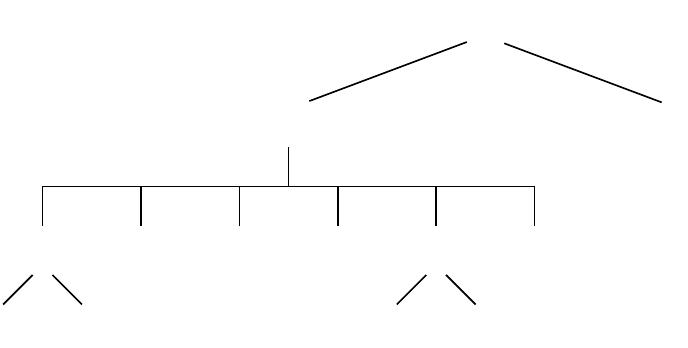 \hfill
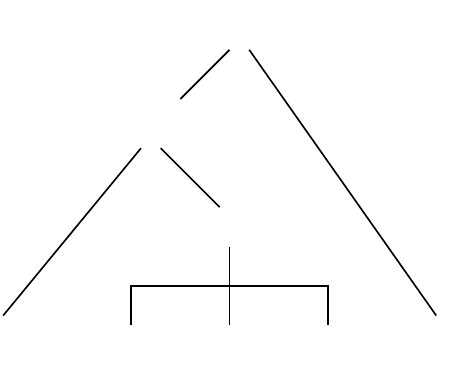
\caption{Simon factorization tree} \label{fig:simon-tree}
\end{figure}

Let $m_i$ be the monoid element obtained at the root of the tree $\tree_i$. We obtain the sequence 
\[
[m_0, h(\gamma_1), m_1, h(\gamma_2), \ldots, m_{t-1}, h(\gamma_t), m_t]
\] 
of monoid elements alternating between the roots of the trees $\tree_i$ and the elements corresponding to the modified positions in $B$. 
Then we compute a Simon $h$-factorization tree $\tree'$ for this sequence $[m_0, h(\gamma_1), m_1, h(\gamma_2), \ldots, m_{t-1}, h(\gamma_t), m_t]$ according to Theorem~\ref{thm:simonFactorization}.
We can now plug in the trees $\tree_i$ at the corresponding leafs of $\tree'$ for $m_i$. This results in a Simon $h$-factorization tree for $B_T$. 

The complexity claims follow from the complexities in Theorem~\ref{thm:simonFactorization} and the fact that the height of the trees $\tree_i$ is linear in the height of $\tree_B$.
\end{proof}

Combining Theorem~\ref{thm:simonFactorization} with Lemma~\ref{lem:compute-parameters} and \ref{lem:compute-training-factorization}, we can build a learning algorithm as claimed in Theorem~\ref{the:precomputation-unary}.
\begin{itemize}\itemsep-2pt
\item \emph{Indexing Phase:} For a string $B$, compute a Simon $h$-factorization $\tree_B$ of $B_\emptyset$ according to Theorem~\ref{thm:simonFactorization}.
\item \emph{Learning Phase:} For a given Training set $T$, compute a Simon $h$-factorization tree $\tree$ of $B_T$ according to Lemma~\ref{lem:compute-training-factorization}, and then compute a consistent set of parameters according to Lemma~\ref{lem:compute-parameters}. 
\end{itemize}
The claimed complexities follow from the ones in Theorem~\ref{thm:simonFactorization} and Lemma~\ref{lem:compute-parameters} and \ref{lem:compute-training-factorization}.
This finishes the proof of Theorem~\ref{the:precomputation-unary}.

\paragraph{Formulas of higher arity.} The methods developed in this section for unary MSO formulas can be adapted to some extent to MSO formulas of higher arity. However, the learning time of this adapted algorithm is not linear in the size $|T|$ anymore.

\begin{theorem} \label{the:precomputation-higherArity}
There is a consistent parameter learner for MSO formulas with indexing time $O(|B|)$ for a string $B$ and learning time $O((k|T|)^\ell)$ for a training set $T$ over $B$, and $\ell$ the number of parameters. 
\end{theorem}

The main difference is that it is not possible to encode a complete training set for examples of higher arity by an annotation of the string $B$. For this reason, one has to do an exhaustive search over the possible parameter positions relative to the positions appearing in the training set. Again based on a factorization tree, one can check for each such relative positioning if concrete parameters with these relative positions exist that are consistent with the training set. If they exist, one can synthesize them with the same idea as for Algorithm~\ref{alg:traversalBinaryTreeUnaryRelation}. We save the details for a full version of this article.

\section{Conclusions}

We study the learnability results for MSO-definable hypotheses over string data. The key question we ask is whether learning is possible in time independent of (or at least sublinear in) the size of the background string. We prove that this is only possible if we allow to build an index of the string first, in time linear in the size of the string.

It is an interesting open question whether our results can be extended to tree-structured data (such as XML-documents). Note that there is no direct generalization of the factorization forests for trees.

\end{document}